\documentclass{article}

\usepackage{microtype}
\usepackage{graphicx}
\usepackage{booktabs}
\usepackage{hyperref}

\usepackage[accepted]{icml2026}
\usepackage{amsmath, amssymb, amsthm}
\usepackage{enumitem}
\usepackage{tikz}
\usepackage{array}
\usetikzlibrary{trees,positioning}

\makeatletter
\providecommand{\@LN}[2]{}
\providecommand{\@LN@col}[1]{}
\makeatother

\newcommand{\E}{\mathbb{E}}

\newcommand{\cS}{\mathcal{S}}
\newcommand{\cA}{\mathcal{A}}

\newcommand{\safeincludegraphics}[2][]{%
  \IfFileExists{#2}{%
    \includegraphics[#1]{#2}%
  }{%
    \fbox{\parbox[c][0.32\linewidth][c]{\dimexpr\linewidth-2\fboxsep-2\fboxrule\relax}{%
      \centering\small Missing figure\\[0.4em]\raggedright\scriptsize\nolinkurl{#2}%
    }}%
  }%
}

\newtheorem{definition}{Definition}[section]
\newtheorem{proposition}{Proposition}[section]
\newtheorem{assumption}{Assumption}[section]
\newtheorem{lemma}{Lemma}[section]
\newtheorem{corollary}{Corollary}[section]

\title{Drowning in Routine: \\ Signal Dilution in Multi-Turn Agent Training}
\icmltitlerunning{Drowning in Routine: Signal Dilution in Multi-Turn Agent Training}

\makeatletter
\renewcommand{\ICML@appearing}{\textit{Accepted at the FAGEN Workshop at the $\mathit{43}^{rd}$ International Conference on Machine Learning}, Seoul, South Korea, 2026. Copyright 2026 by the author(s).}
\makeatother

\begin{document}
\twocolumn[
  \icmltitle{Drowning in Routine: \\ Signal Dilution in Multi-Turn Agent Training}

  \begin{icmlauthorlist}
    \icmlauthor{Yann Pernot}{mila,poly}
    \icmlauthor{Vi Retault}{poly}
  \end{icmlauthorlist}

  \icmlaffiliation{mila}{Mila - Qu\'ebec AI Institute}
  \icmlaffiliation{poly}{Polytechnique Montr\'eal}

  \icmlcorrespondingauthor{Yann Pernot}{yann.pernot@mila.quebec}

  \vskip 0.3in
]

\printAffiliationsAndNotice{}

\begin{abstract}
Multi-turn agents interleave consequential decisions with routine execution: some actions change the downstream return distribution, while others are necessary but reward-equivalent. The cost of trajectory-level credit assignment, often attributed to long horizons, is in fact governed by \emph{decision density} $\rho$: the fraction of turns whose actions affect the return. When decision density is low, routine turns create \emph{signal dilution}: they add gradient variance to trajectory-level estimators such as GRPO without adding expected signal. Under explicit assumptions, the resulting turn-level to trajectory-level signal-to-noise ratio scales as $\rho^{-1/2}$, provided critic error remains controlled. The same analysis identifies the complementary regime: at high decision density, trajectory-level methods can remain competitive while avoiding the cost of a critic. In a controlled environment where $\rho$ is exactly tunable, the predicted scaling is recovered with $R^2 = 0.999$, and the training-step gap widens significantly as $\rho \to 0$.
\end{abstract}

\section{Introduction}

Multi-turn reinforcement learning faces a credit-assignment problem: a single terminal reward must be allocated across many actions, only some of which actually determined the outcome \citep{minsky1961steps,seo2019rewards}. The textbook remedy is to learn a value function and compute per-turn advantages, as PPO \citep{schulman2017ppo} does. For LLM-scale policies, a learned critic is expensive: it doubles the model footprint and is hard to train accurately at scale \citep{shao2024deepseekmath}. Critic-free trajectory-level methods such as Group Relative Policy Optimization (GRPO) \citep{shao2024deepseekmath} are therefore attractive in single-turn LLM post-training, but the same choice can become costly in multi-turn agentic settings. Recent work reports instability of direct multi-turn GRPO \citep{li2026turnppo} and gains from finer-grained credit assignment on multi-turn agent benchmarks \citep{feng2025gigpo, wei2025reinforcing}. These results point to a recurring empirical gap without isolating its mechanism or the task properties that govern it. We address this by identifying a signal dilution mechanism behind the trajectory-level penalty and tying its strength to a structural quantity: the decision density.

Real multi-turn agents tend to spend most of their turns on actions that are necessary but not in themselves consequential \citep{zhou2024webarena}. Consider a coding agent debugging a repository \citep{yang2024sweagent}. A handful of turns correspond to task-level commitments: deciding whether the bug is in the parser or the optimizer, choosing which file to patch, or selecting one fix rather than another. These turns significantly impact the rest of the trajectory and therefore the future return distribution. We call them \emph{critical}. By contrast, some turns occur after such a subgoal has been fixed. If the agent has already decided to inspect the test failures, then running the full suite, running the failing test directly, or inspecting a cached log are different actions that serve the same purpose and tend to leave the downstream probability of solving the task nearly unchanged. We call them \emph{routine}. We argue that the fraction of trajectory turns that are critical, which we call the \emph{decision density} $\rho$, is a useful structural variable for reasoning about the cost of going critic-free. 

Looking at the problem through this lens makes the failure mode of trajectory-level training precise. We formalize idealized routine states through distributional invariance of the episode return, and show that every routine turn injects variance into a trajectory-level gradient while contributing no signal in expectation; a disadvantage the turn-level estimator avoids by zeroing routine contributions pointwise. We call this mechanism \emph{signal dilution}. We use signal-to-noise ratio (SNR) as a training-step efficiency proxy, following prior work that relates policy-gradient SNR to learning performance under a fixed evaluation budget \citep{roberts2008snr}. Under non-degeneracy assumptions, the turn-level SNR advantage scales as $\Theta(\rho^{-1/2})$ in the perfect critic regime, yielding a theoretical upper envelope on the gains a critic can provide. The analysis extends beyond the perfect critic case: turn-level gains survive only while critic-error variance stays below decision density, and saturate once it exceeds it.

We instantiate the assumptions in \emph{Diluted Doors}, a controlled MDP whose decision density is exactly tunable; Exp.~1 quantifies the predicted gap, and Exp.~2 quantifies the resulting training-step efficiency gap. The predicted $\rho^{-1/2}$ law explains $99.9\%$ of the variation in the empirical SNR ratio over a fifty-fold range of $\rho$,  and the training-step gap between the two estimators widens significantly as $\rho \to 0$.

\paragraph{Contributions.}
We (i) formalize routine and critical turns and define decision density $\rho$; (ii) identify signal dilution as the mechanism behind the trajectory-level penalty: turn-level estimators cancel routine contributions pointwise, trajectory-level estimators only in expectation, leaving excess variance; (iii) derive the SNR scaling $\Theta(\rho^{-1/2})$ and extend it to imperfect critics; (iv) reframe trajectory-level cost as governed by decision density rather than horizon alone, with $\rho^{-1/2}$ as a theoretical upper envelope; (v) quantify in a controlled MDP both the predicted SNR gap and the training-step efficiency gap, which widens as $\rho \to 0$.

\section{Preliminaries}
\label{sec:preliminaries}

\subsection{Environment model}

We model the environment the agent evolves in as a finite-horizon MDP with a terminal reward \citep{sutton2018reinforcement}, defined as a tuple $\mathcal{M} = (\cS, \cA, P, R, T)$. $\cS$ is the state space, $\cA$ is the action space, $P: \cS \times \cA \to \Delta(\cS)$ is the state transition kernel, $R: \cS_T \to [0, 1]$ is the terminal reward, defined on states reached at horizon $T$, the number of turns. In this formulation, an LLM is not itself the MDP; it is the stochastic policy embedded in an interaction process whose state evolves over turns.

For notational simplicity, we write \(s_t\) for the agent-facing decision state, including all information needed by the policy.\footnote{In an LLM agent, \(s_t\) can be taken to include the interaction history or rendered prompt, so the policy is written as \(\pi_\theta(a_t \mid s_t)\). If the rendered history is only a partial observation of hidden environment state, the original problem is a POMDP; using the full history yields the standard history-state MDP view.} The policy is assumed to be differentiable in $\theta$ for every $(s,a)$ with $\pi_\theta(a\mid s)>0$.

\subsection{Trajectory-level and turn-level estimators}
\label{sec:estimators}

We compare two families of policy-gradient estimators \citep{sutton2000policy} that differ in how they assign credit to individual turns.

Consider a trajectory $\tau = (s_0, a_0, \ldots, s_{T-1}, a_{T-1})$.
 
A \emph{turn-level} policy-gradient estimator uses a state-value function $V$ to compute a \emph{per-turn advantage} $\hat{A}_t$ at every step \citep{schulman2015gae}:
\begin{equation}
  \hat{g}_{\mathrm{turn}} = -\sum_{t=0}^{T-1} \hat{A}_t \, \nabla_\theta \log \pi_\theta(a_t \mid s_t),
  \label{eq:ppo-grad}
\end{equation}
With a perfect critic, $\hat{A}_t = A^{\pi_\theta}(s_t, a_t) = Q^{\pi_\theta}(s_t, a_t) - V^{\pi_\theta}(s_t)$, so the credit assigned to turn $t$ directly reflects how much that action changed the expected return. PPO is an instantiation of this family \citep{schulman2017ppo}.

A \emph{trajectory-level} policy-gradient estimator eliminates the value function by attaching a single scalar advantage $\hat{A}^{(i)}$ to every turn of trajectory $\tau^{(i)}$:
\begin{equation}
  \hat{g}_{\mathrm{traj}} = -\hat{A}^{(i)} \sum_{t=0}^{T-1} \nabla_\theta \log \pi_\theta(a_t^{(i)} \mid s_t^{(i)}).
  \label{eq:grpo-grad}
\end{equation}
GRPO is an instantiation \citep{shao2024deepseekmath}: it samples $G$ trajectories from the same policy and uses the group-normalized return
\begin{equation}
\hat{A}^{(i)} = \frac{R(\tau^{(i)}) - \bar{R}}{\sigma_R +\epsilon}\label{eq:grpo-advantage}
\end{equation}
where $\bar{R}$ and $\sigma_R$ are the mean and standard deviation of returns within the group. Other trajectory-level estimators (e.g., REINFORCE with a baseline \citep{greensmith2004variance,williams1992reinforce}, RLOO \citep{ahmadian2024back}) differ in how $\hat{A}^{(i)}$ is built from the group's returns but share the structure of Eq.~\eqref{eq:grpo-grad}: a single scalar advantage applied uniformly across the trajectory's scores.

The key structural contrast is that \textbf{the turn-level advantage $\hat{A}_t$ is specific to each turn}, while \textbf{the trajectory-level advantage $\hat{A}^{(i)}$ is a single scalar applied identically to every turn} of trajectory $\tau^{(i)}$.

In practice, instantiations such as PPO and GRPO optimize a clipped surrogate, and may include a KL penalty against a reference policy. We consider the simplified gradient for generality and to isolate the mechanism.

\noindent Throughout the analysis, we assume trajectory-level advantages are uniformly bounded, $|\hat A^{(i)}|\le M_A<\infty$ a.s.

\subsection{Gradient Signal-to-Noise Ratio}
\label{sec:snr}

Policy-gradient estimators are stochastic as they are computed from sampled rollouts. For a fixed parameter coordinate $k$, the parameter-wise (coordinate-wise) SNR compares the expected update $\E[\hat g_k]$ (the signal) with the standard deviation of the sampled update $\hat g_k$ (the noise).

This is the square root of the parameter-wise GSNR used by \citet{liu2020gsnr}, and a coordinate-wise analogue of the policy-gradient SNR studied by \citet{roberts2008snr}, who showed empirically that SNR tracks learning performance under a fixed evaluation budget.

\begin{definition}[Parameter-wise gradient SNR]
\label{def:snr}
For a gradient estimator $\hat g$ with $0 < \mathrm{Var}[\hat g_k] < \infty$, the signal-to-noise ratio of parameter $k$ is
\[
  \mathrm{SNR}_k(\hat g)
  \;=\;
  \frac{|\E[\hat g_k]|}{\sqrt{\mathrm{Var}[\hat g_k]}}.
\]
\end{definition}

\noindent The SNR is invariant to the sign of $\hat g$, so we omit the leading $-$ from Eqs.~\eqref{eq:ppo-grad}--\eqref{eq:grpo-grad} hereafter.

The parameter-wise SNR is the object of the theoretical analysis. To summarize full-gradient estimates in experiments, we also use a pooled SNR, defined as the mean-gradient norm divided by the aggregate coordinate standard deviation:

\begin{definition}[Pooled gradient SNR]
\label{def:snr-pooled}
For a gradient estimator $\hat g$ with $0<\sum_k\mathrm{Var}[\hat g_k]<\infty$, the pooled (full-gradient) signal-to-noise ratio is
\[
  \mathrm{SNR}(\hat g)
  \;=\;
  \frac{\|\E[\hat g]\|}{\sqrt{\textstyle\sum_k \mathrm{Var}[\hat g_k]}}.
\]
\end{definition}

\section{The Signal Dilution Problem}
\label{sec:dilution}

\subsection{Critical States, Routine States, and Decision Density}

In real trajectories, most turns offer action choices that are interchangeable with respect to the final return; routine states formalize that. \footnote{Distributional, rather than just expected, invariance is needed because trajectory-level estimators apply $\hat A^{(i)}$, a generally non-linear function of the realized return, to every score; equal expected returns across actions do not imply equal $\E[\hat A\mid s,a]$ when $\hat A$ is non-linear in $R$.}

\begin{definition}[Routine state]
\label{def:routine}
Given a policy $\pi$, let $\mathcal{U}_\pi(s) = \{a \in \cA : \pi(a \mid s) > 0\}$ denote its effective action support in state $s$. A state $s \in \cS$ is \emph{routine} (under $\pi$) if the conditional distribution\footnote{$\mathcal{L}(X \mid \cdot)$ denotes the conditional distribution (law) of $X$.} of the episode return is the same for every action in this support:
\[
  \begin{aligned}
    &\mathcal{L}\!\left(R(\tau) \;\middle|\; s_t = s,\, a_t = a,\, \pi\right)\\
    &\qquad =
    \mathcal{L}\!\left(R(\tau) \;\middle|\; s_t = s,\, a_t = a',\, \pi\right)\\
    &\qquad\qquad \forall a, a' \in \mathcal{U}_\pi(s).
  \end{aligned}
\]

A state $s \in \cS$ is \emph{critical} (under $\pi$) if it is not routine. We denote by $\cS_R \subseteq \cS$ the set of routine states under $\pi$ and by $\cS_C = \cS \setminus \cS_R$ the set of critical states under $\pi$.
\end{definition}

\noindent Definition~\ref{def:routine} describes an idealization: in real environments, exact distributional invariance is rare and routine-ness lies on a spectrum. We adopt the exact partition because it admits clean statements and enables a controlled test of the prediction; we expect the scaling to hold approximately when states are approximately routine, and return to this point in the discussion.

\noindent Decision density summarizes how concentrated the learning-relevant decisions are along the trajectory.

\begin{definition}[Decision density]
\label{def:decision-density}
Let $C^{(i)}$ denote the number of critical states visited by trajectory $i$, $\bar{C} = \E[C^{(i)}]$ the expected count. The \emph{decision density} (under $\pi$) is $\rho = \bar{C} / T$.
\end{definition}

\noindent When $\rho$ is high, most turns can affect the return; when $\rho$ is low, many turns are routine and can dilute a trajectory-level learning signal.

\subsection{Critical/Routine Gradient Decomposition}
\label{sec:decomposition}

Let $u_t^{(i)} = \nabla_\theta \log \pi_\theta(a_t^{(i)} \mid s_t^{(i)})$ denote the score function at turn $t$ of trajectory $i$. The trajectory-level gradient splits directly into critical-state and routine-state terms:
\begin{equation}
  \hat{g}_{\mathrm{traj}}
  =
  \underbrace{\hat{A}^{(i)}
  \sum_{t:\,s_t^{(i)}\in\cS_C} u_t^{(i)}}_{\hat g_C\;\text{(critical)}}
  \;+\;
  \underbrace{\hat{A}^{(i)}
  \sum_{t:\,s_t^{(i)}\in\cS_R} u_t^{(i)}}_{\hat g_R\;\text{(routine)}} .
  \label{eq:decomp}
\end{equation}

The critical term carries the learning signal. The routine term is the source of signal dilution: as we show next, it has zero expectation but nonzero variance for trajectory-level estimators.

\subsection{Routine Signal Vanishing}
\label{sec:signal-equiv}

The following proposition establishes that, for both estimators, the expected gradient receives contributions only from critical states.

\noindent We first establish a lemma that is the engine of the proof. The expectation throughout is taken over the $G$ trajectories $\tau^{(1)}, \ldots, \tau^{(G)}$ sampled independently from $\pi_\theta$.

\begin{lemma}[Vanishing score expectation at routine states]
\label{lem:vanishing-score}
Let \(F^{(i)}=\psi_i(R^{(1)},\ldots,R^{(G)})\) be a return-only trajectory scalar, with $\psi_i$ and \(F^{(i)}u_t^{(i)}\) integrable. Assume that the policy support is locally constant in $\theta$.

 Then,
\[
  \E[F^{(i)}u_t^{(i)} \mid s_t^{(i)}=s]=\mathbf{0}
  \quad \forall s\in\cS_R.
\]
\end{lemma}

\begin{proof}
We prove the Lemma in the discrete action space setting.

\noindent Fix a rollout $i$, a turn $t$, and a state $s\in\cS_R$. Write $\tau^{(-i)}=\{\tau^{(j)}\}_{j\neq i}$.

\noindent By the law of iterated expectations, $F^{(i)} u_t^{(i)}$ being integrable,
\[
  \begin{aligned}
    &\E\!\left[F^{(i)} u_t^{(i)} \mid s_t^{(i)}=s\right]\\
    &\qquad =
    \E_{\tau^{(-i)} \mid s_t^{(i)}=s}\!\left[
      \E\!\left[F^{(i)} u_t^{(i)}
      \;\middle|\; s_t^{(i)}=s, \tau^{(-i)}\right]
    \right].
  \end{aligned}
\]
For fixed $\tau^{(-i)}$, the returns $R^{(j)}$ for $j\neq i$ are fixed. Define
\[
  \psi_{i,\tau^{(-i)}}(r)
  =
  \psi_i\!\left(
    R^{(1)},\ldots,R^{(i-1)},r,R^{(i+1)},\ldots,R^{(G)}
  \right).
\]
For any supported action $a$, set
\[
  m_a
  :=
  \E\!\left[
    \psi_{i,\tau^{(-i)}}(R^{(i)})
    \;\middle|\; s_t^{(i)}=s, a_t^{(i)}=a, \tau^{(-i)}
  \right].
\]
Since $s$ is routine, and $\tau^{(-i)}$ is independent of rollout $i$, Definition~\ref{def:routine} gives
\[
  \begin{aligned}
    &\mathcal{L}\!\left(R^{(i)} \mid s_t^{(i)}=s, a_t^{(i)}=a, \tau^{(-i)}\right)\\
    &\qquad =
    \mathcal{L}\!\left(R^{(i)} \mid s_t^{(i)}=s, a_t^{(i)}=a', \tau^{(-i)}\right)\\
    &\qquad\qquad \forall a, a' \in \mathcal{U}_\pi(s)
  \end{aligned}
\]

\noindent Therefore, by integrability of $\psi_i$,
\[
m_a=m_{a'}=c(s,\tau^{(-i)}) \quad \forall a, a' \in \mathcal{U}_\pi(s)
\]

\noindent Then
\[
\begin{aligned}
  &\E\!\left[F^{(i)} u_t^{(i)}
    \;\middle|\; s_t^{(i)}=s, \tau^{(-i)}\right]\\
  &=
  \sum_{a\in\mathcal{U}_\pi(s)}
  \Pr(a_t^{(i)}=a \mid s_t^{(i)}=s, \tau^{(-i)})\\
  &\quad{}\times
  \E\!\left[
    F^{(i)}u_t^{(i)}
    \;\middle|\; s_t^{(i)}=s, a_t^{(i)}=a, \tau^{(-i)}
  \right] \\
  &=
  \sum_{a\in\mathcal{U}_\pi(s)}
  \pi_\theta(a\mid s)
  \nabla_\theta\log\pi_\theta(a\mid s)\\
  &\quad{}\times
  \E\!\left[
    F^{(i)}
    \;\middle|\; s_t^{(i)}=s, a_t^{(i)}=a, \tau^{(-i)}
  \right] \\
  &\hspace{1em}\text{($u_t^{(i)}$ fixed given $s,a$)}\\
  &=
  \sum_{a\in\mathcal{U}_\pi(s)}
  \pi_\theta(a\mid s)
  \nabla_\theta\log\pi_\theta(a\mid s)\,m_a \\
  &=
  c(s,\tau^{(-i)})
  \sum_{a\in\mathcal{U}_\pi(s)}
  \nabla_\theta\pi_\theta(a\mid s) \\
  &=
  c(s,\tau^{(-i)})
  \nabla_\theta
  \sum_{a\in\mathcal{U}_\pi(s)}\pi_\theta(a\mid s) \\
  &= \mathbf{0}.\\
  &\hspace{1em}\text{(locally constant support)}
\end{aligned}
\]
Taking the outer expectation gives the claim.
\end{proof}

\noindent Lemma~\ref{lem:vanishing-score} can be directly used for the following proposition.

\begin{proposition}[Routine Signal Vanishing]
\label{prop:signal-equiv}
Under the assumptions of Lemma~\ref{lem:vanishing-score} (trajectory-level only), the expected gradient of both estimators receives contributions only from critical states:
\begin{enumerate}[nosep, label=(\roman*)]
  \item \textbf{Trajectory-level estimator.} The gradient contribution at routine states vanishes in expectation:
  \[
    \begin{aligned}
      &\E\!\left[
        \hat{A}^{(i)} \cdot u_t^{(i)}
        \;\middle|\; s_t^{(i)}=s
      \right]\\
      &\qquad = \mathbf{0}
      \quad \forall s\in\cS_R,\\
      \text{so}\qquad
      \E[\hat{g}_{\mathrm{traj}}]
      &=
      -\sum_{t=0}^{T-1}
      \E\!\left[\hat{A}\cdot u_t\,\mathbf 1\{s_t\in\cS_C\}\right].
    \end{aligned}
    \label{eq:grpo-signal}
  \]
  \item \textbf{Turn-level estimator (perfect critic, $\hat{A}_t = A^{\pi_\theta}(s_t, a_t)$).} The gradient contribution at routine states vanishes deterministically:
  \[
    \begin{aligned}
      \hat{A}_t \cdot u_t
      &= \mathbf{0}
      \quad \text{a.s.\ for all } t \text{ with } s_t \in \cS_R,\\
      \text{so}\qquad
      \E[\hat{g}_{\mathrm{turn}}]
      &=
      -\sum_{t=0}^{T-1}
      \E\!\left[\hat{A}_t\cdot u_t\,\mathbf 1\{s_t\in\cS_C\}\right].
    \end{aligned}
    \label{eq:ppo-signal}
  \]
\end{enumerate}
\end{proposition}

\begin{proof}
\emph{Part (i).} Lemma~\ref{lem:vanishing-score} gives $\E[\hat{A}^{(i)} u_t^{(i)}] = \mathbf{0}$ at every routine step. The claim on $\E[\hat{g}_{\mathrm{traj}}]$ follows by linearity of expectation.

\emph{Part (ii).} By Definition~\ref{def:routine}, for $s_t\in\cS_R$, the expected return is the same for every
$a\in\mathcal{U}_\pi(s_t)$. Therefore
$Q^{\pi_\theta}(s_t,a)=V^{\pi_\theta}(s_t)$ for all supported actions,
so $A^{\pi_\theta}(s_t,a)=0$.  The claim on $\E[\hat{g}_{\mathrm{turn}}]$ again follows by linearity of expectation.
\end{proof}

\paragraph{Interpretation.} Proposition~\ref{prop:signal-equiv} pins down the \textbf{source of signal dilution}: for both estimators the expected gradient comes entirely from critical states, but the two families cancel routine contributions differently. The turn-level estimator zeros them \emph{pointwise} ($\hat A_t u_t = \mathbf 0$ at every routine step), while the trajectory-level estimator zeros them only \emph{in expectation} ($\E[\hat A\, u_t]=\mathbf 0$, but $\hat A\, u_t\neq \mathbf 0$ on individual trajectories). The residual per-trajectory routine noise is what produces the trajectory-level estimator's excess variance, and it leads to a structural gap (Section~\ref{sec:structural}).

\subsection{Structural SNR Scaling}
\label{sec:structural}
We now quantify the cost of the variance gap identified in Proposition~\ref{prop:signal-equiv}. Throughout the section, we consider the fixed-$\bar C$ sweep $T = \bar C / \rho \to \infty$ (the number of critical turns is held fixed; routine turns are added). Under Assumptions~\ref{as:a1}--\ref{as:a2} and a perfect critic, the turn-level to trajectory-level SNR ratio admits a closed-form scaling in the decision density $\rho$, which diverges as $\rho \to 0$.

The first proposition is coordinate-wise. Fix a parameter coordinate $k$ and write
\[
  \begin{aligned}
    \Phi_{C,k}
    &= \sum_{t:\,s_t\in\cS_C} u_{t,k},
    \quad
    \Phi_{R,k}
    = \sum_{t:\,s_t\in\cS_R} u_{t,k}.
  \end{aligned}
\]
\[
  \Phi_{k} = \Phi_{C,k} + \Phi_{R,k}.
\]

The result rests on two assumptions: a regularity condition keeping the critical sub-problem well-posed as $T$ grows (A1), and a substantive non-degeneracy condition on routine exploration (A2).

\begin{assumption}[A1 -- Stable critical sub-problem]
\label{as:a1}
The critical-state contributions to the expected gradient are bounded: there exist constants $0<m_\mu\le M_\mu<\infty$ such that
\[
  m_\mu \;\le\; |\E[\hat g_{\mathrm{turn},C,k}]|,\;|\E[\hat g_{\mathrm{traj},C,k}]| \;\le\; M_\mu,
\]
and the corresponding critical variances are bounded: there exists $0<m_C\leq M_C<\infty$ with
\[
  m_C\leq\mathrm{Var}[\hat g_{\mathrm{turn},C,k}]\le M_C,
  \qquad
  \mathrm{Var}[\hat A\,\Phi_{C,k}]\le M_C .
\]
\end{assumption}

A1 fixes the scale of the critical problem as routine turns are added: the reward-relevant signal neither vanishes nor blows up, and critical variance stays bounded for both estimators. The underlying critical task does not change as we scale $T$.

\begin{assumption}[A2 -- Non-degenerate routine exploration and score accumulation]
\label{as:a2}
Let $N_R$ be the routine-state count of a trajectory. Let $\mathcal{H}_C$ be the sigma-field generated by $N_R$, the critical trajectory information, and the group's terminal returns, and assume $\hat A$ is $\mathcal{H}_C$-measurable. There exist constants $c_N, c_A, c_R > 0$ and $C_R < \infty$ such that
\begin{enumerate}[label=(\alph*)]
  \item $\E[N_R] \ge c_N T$;
  \item $\E[\hat A^2\, N_R] \ge c_A\,\E[N_R]$;
  \item $c_R\, N_R \;\le\; \E[\Phi_{R,k}^2 \mid \mathcal{H}_C] \;\le\; C_R\, N_R$ a.s.
\end{enumerate}
\end{assumption}

A2 imposes three conditions: (a) the routine-state count grows linearly in $T$; (b) the trajectory advantage doesn't collapse on routine-heavy trajectories; (c) conditional routine score energy grows proportionally to the routine-state count. (a) is bookkeeping. The upper half of (c) rules out coherent long-range alignment of routine scores and rarely fails outside adversarial constructions. The load-bearing conditions are (b) and the lower half of (c). (b) is specific to group-relative advantages: it fails when all rollouts in a group share the same return, a known GRPO failure mode tied to task difficulty and group size \citep{le2026noprompt}. The lower half of (c) is the broader concern: it fails when the policy becomes nearly deterministic on routine turns, which entropy collapse pushes toward in late training \citep{cui2025entropy}.

\begin{lemma}[Critical/routine bounds at coordinate $k$]
\label{lem:critical-routine-bounds}
Assume a perfect critic and the standing bounded-advantage condition, and fix a coordinate $k$.
\begin{enumerate}[label=(\roman*)]
  \item Under the upper-bound parts of A1 and the upper half of A2(c),
  \[
    |\E[\hat g_{\mathrm{turn},k}]|,\; |\E[\hat g_{\mathrm{traj},k}]| = O(1),
  \]
  \vspace{-15pt}
  \[
    \mathrm{Var}[\hat g_{\mathrm{turn},k}] = O(1),
    \quad
    \mathrm{Var}[\hat g_{\mathrm{traj},k}] = O(T).
  \]
  \item If, in addition, the lower-bound parts of A1, A2(a), A2(b), and the lower half of A2(c) hold at $k$, the matching $\Omega(\cdot)$ bounds hold.
\end{enumerate}
Proof in Appendix~\ref{app:structural-proof}.
\end{lemma}

\begin{proposition}[Structural SNR scaling]
\label{prop:structural-scaling}
Under the assumptions of Lemma~\ref{lem:critical-routine-bounds}(ii),
\[
  \frac{\mathrm{SNR}_k(\hat g_{\mathrm{turn}})}{\mathrm{SNR}_k(\hat g_{\mathrm{traj}})} \;=\; \Theta(\rho^{-1/2}).
\]
Proof in Appendix~\ref{app:structural-proof}.
\end{proposition}

\paragraph{Sketch.} Both expected gradients receive contributions only from critical states (Prop.~\ref{prop:signal-equiv}) and are bounded above and below by A1. The turn-level estimator zeros routine terms pointwise, so its variance stays $\Theta(1)$. The trajectory-level estimator instead pays $\hat A$ times the accumulated routine score, whose variance is $\Theta(T)$ under A2. Combining gives $\mathrm{SNR}_{\mathrm{turn}}/\mathrm{SNR}_{\mathrm{traj}} = \Theta(\rho^{-1/2})$.

\paragraph{Upper-envelope reading.} The substantive parts of A2, A2(b) and the lower half of A2(c), enter only on the $\Omega(\rho^{-1/2})$ side. When they fail, $\mathrm{Var}[\hat g_{\mathrm{traj}}]$ grows slower than $T$ and the realized ratio falls below $\rho^{-1/2}$, while the $O(\rho^{-1/2})$ upper bound still holds. Two further idealizations point the same way. Exact routine states let turn-level zero routine-turn contributions pointwise; approximate ones leave residual variance. A perfect critic avoids the routine-turn variance that critic error would otherwise reintroduce into turn-level (cf.\ Prop.~\ref{prop:imperfect-critic-scaling}). The $\rho^{-1/2}$ scaling can therefore be read as an upper envelope on the SNR gap achievable in real agentic tasks.

\paragraph{Decision density, not horizon.} The cost of trajectory-level credit assignment relative to turn-level is often attributed to horizon length. The variance gap between estimators is generated by routine turns: by Proposition~\ref{prop:signal-equiv}, routine turns add variance to the trajectory-level estimator only, while the turn-level estimator is unaffected. Trajectory length matters only insofar as it controls the routine-turn count.

The ratio diverges as $\rho \to 0$. The pooled SNR inherits the same scaling under weaker assumptions: the lower bounds need hold at only one coordinate.

\begin{corollary}[Pooled SNR scaling]
\label{cor:pooled-snr}
Assume a perfect critic, fixed parameter dimension $d$, and that A2(a)--(b) hold. Suppose the upper-bound parts of A1 and the upper half of A2(c) hold uniformly at every coordinate $k = 1, \ldots, d$, and that the lower-bound parts of A1 together with the lower half of A2(c) hold at some coordinate $k^\star$. Then
\[
  \frac{\mathrm{SNR}(\hat g_{\mathrm{turn}})}{\mathrm{SNR}(\hat g_{\mathrm{traj}})} = \Theta(\rho^{-1/2}).
\]
Proof in Appendix~\ref{app:pooled-proof}.
\end{corollary}

\subsection{Imperfect Critic Regime}
\label{sec:imperfect-critic}

The previous scaling describes the best-case turn-level estimator, where a perfect critic removes routine-state contributions pointwise. We now ask what remains when the turn-level advantage is noisy. Write the imperfect critic estimate as
$\widetilde A_t = A^{\pi_\theta}(s_t,a_t) + \epsilon_t$,
and define
\[
  \hat g_{\mathrm{turn},\epsilon}
  =
  \sum_{t=0}^{T-1}
  \widetilde A_t u_t .
\]
Equivalently,
\[
  \hat g_{\mathrm{turn},\epsilon}
  =
  \underbrace{
  \sum_{t=0}^{T-1}
  A^{\pi_\theta}(s_t,a_t)u_t
  }_{\hat g_{\mathrm{turn}}^\star}
  +
  \underbrace{
  \sum_{t=0}^{T-1}
  \epsilon_t u_t
  }_{Z_\epsilon},
\]
where $\hat g_{\mathrm{turn}}^\star$ is the perfect critic turn-level estimator and $Z_\epsilon$ is the critic-error contribution.

We make the assumption directly on the gradient contribution of the critic error.

\begin{assumption}[A3 -- Gradient-unbiased critic error]
\label{as:a3}
Let $\mathcal F_\tau$ be the sigma-field generated by the sampled trajectory. The critic-error contribution is conditionally mean-zero,
\[
  \E[Z_\epsilon \mid \mathcal F_\tau] = \mathbf 0,
\]
and its coordinate-wise critic-noise level
\[
  \sigma_{\epsilon,k}^2(T)
  :=
  \frac{1}{T}\,\E\!\left[\mathrm{Var}(Z_{\epsilon,k} \mid \mathcal F_\tau)\right]
\]
is finite for every $k$.
\end{assumption}

Assumption~\ref{as:a3} models imperfect critics through a variance-only error channel. The finite-variance condition is mild: it requires the critic-error contribution to grow at most linearly with horizon. The conditional mean-zero condition is the substantive part: it ensures that critic error preserves the expected policy-gradient signal and enters only as additional stochastic gradient noise around the perfect-critic update. A deterministic learned critic may also introduce systematic approximation bias; such bias is outside Proposition~\ref{prop:imperfect-critic-scaling}, and can misdirect the update rather than lowering its SNR.

\begin{proposition}[Imperfect critic SNR scaling]
\label{prop:imperfect-critic-scaling}
Fix a coordinate $k$ satisfying the standing bounded-advantage condition and Assumptions~\ref{as:a1}--\ref{as:a3}. Then
\[
  \frac{
    \mathrm{SNR}_k(\hat g_{\mathrm{turn},\epsilon})
  }{
    \mathrm{SNR}_k(\hat g_{\mathrm{traj}})
  }
  =
  \Theta\!\left(
    \frac{1}
    {\sqrt{\rho+\sigma_{\epsilon,k}^2(T)}}
  \right),
\]
with constants depending on $\bar C$ and the constants in Assumptions~\ref{as:a1}--\ref{as:a2}.
Proof in Appendix~\ref{app:imperfect-critic-proof}.
\end{proposition}

\paragraph{Interpretation.}
Critic noise reintroduces routine-turn variance into the turn-level estimator: at routine states, $\widetilde A_t u_t = \epsilon_t u_t$ is mean-zero but no longer pointwise zero. The turn-level gain therefore depends on which dominates, $\rho$ or $\sigma_{\epsilon,k}^2$. When $\sigma_{\epsilon,k}^2\ll\rho$, the perfect-critic scaling $\Theta(\rho^{-1/2})$ is recovered; when $\sigma_{\epsilon,k}^2\gg\rho$, the ratio flattens to a horizontal asymptote $\Theta(\sigma_{\epsilon,k}^{-1})$ set by the critic-error scale, independent of $\rho$. Low decision density makes turn-level credit assignment potentially valuable, but \textbf{only when critic error stays below the decision-density scale}.

\section{Experimental Setup}
\label{sec:setup}
\label{sec:environment}

Our theory predicts that $\rho$ governs the gradient SNR ratio between turn-level and trajectory-level estimators (Section~\ref{sec:dilution}); prior work relates policy-gradient SNR to learning performance under fixed evaluation budgets \citep{roberts2008snr}. Together they motivate two empirical predictions, summarized below.

\vspace{-8pt}

\begin{center}
\begin{tikzpicture}[
    every node/.style={font=\small},
    box/.style={draw=gray!60, rounded corners=2pt, inner sep=5pt, fill=gray!8, minimum height=8mm},
    theory/.style={->, >=stealth, thick, blue!55!black},
    expt/.style={->, >=stealth, thick, red!70!black, dashed},
    edgelabel/.style={font=\scriptsize\itshape, inner sep=2pt}
]
  \node[box] (rho) {$\rho$};
  \node[box, right=2cm of rho] (snr) {$\mathrm{SNR}$};
  \node[box, right=2cm of snr] (sc)  {training-step efficiency};
  \draw[theory] (rho) -- node[edgelabel, above] {Sec.~\ref{sec:dilution}} (snr);
  \draw[theory] (snr) -- node[edgelabel, above] {Sec.~\ref{sec:snr}} (sc);
  \draw[expt] (rho) to[bend right=30] node[edgelabel, below] {Exp.~1} (snr);
  \draw[expt] (rho) to[bend right=35] node[edgelabel, below] {Exp.~2} (sc);
\end{tikzpicture}
\end{center}

\vspace{-8pt}

We test these mechanism-level predictions in \emph{Diluted Doors}, a deterministic MDP where $\rho$ is set exactly and routine transitions are reward-equivalent by construction. Experiment~1 quantifies the gradient SNR ratio at initialization; Experiment~2 quantifies the resulting training-step efficiency gap as a function of $\rho$, with the SNR-to-training-step efficiency link taken from \citet{roberts2008snr}. The policy is a small causal transformer \citep{vaswani2017attention}; full details are in Appendix~\ref{app:experiments}.

\vspace{-8pt}

\paragraph{Environment design.}
Diluted Doors is a balanced tree alternating $C$ critical states (only one of $K_C$ doors is correct, fixed across episodes) and $L$ routine states between consecutive critical states. Every root-to-leaf trajectory has length $T = C(1+L)$ and visits exactly $C$ critical states, so $\rho = 1/(1+L)$ is exact; varying $L$ at fixed $C$ sweeps $\rho$ without changing the number of critical decisions. At each routine state, the $K_R$ outgoing paths lead to isomorphic subtrees, so the distribution of future returns is identical across paths and Definition~\ref{def:routine} holds exactly. Any effect on the SNR ratio is therefore attributable to the routine-state mechanism rather than to hidden signal at routine states.

\vspace{-8pt}

\paragraph{Experiment 1 --- quantifying the predicted SNR gap.}
The two gradient estimators differ only in the advantage term: \emph{trajectory-level} uses a group-relative trajectory scalar (GRPO-like), and \emph{turn-level} uses a per-turn critic-based advantage; at initialization, $V^{\pi_\theta}$ admits a closed-form expression under the near-uniform policy, so the turn-level estimator uses this analytical critic. For each $L$, both estimators are evaluated on the same independently sampled batches across seeds, and the Bessel-corrected pooled SNR is estimated from these batch-gradient samples (Appendix~\ref{app:stats}). We report the paired ratio $R_{\mathrm{SNR}}(\rho) = \mathrm{SNR}_{\mathrm{turn}}/\mathrm{SNR}_{\mathrm{traj}}$, aggregated across seeds by IQM with $95\%$ bootstrap confidence intervals, and fit the power law $R(\rho) = C \rho^\alpha$ in log-log coordinates.

\vspace{-8pt}

\paragraph{Experiment 2 --- SNR-to-training-step efficiency link.}
The turn-level estimator uses a Monte Carlo critic with $k$ rollouts per visited state. This does not fully achieve the perfect-critic regime, nor does it make critic quality constant across $L$, but it gives a controlled critic and avoids learned-critic training dynamics. We report iterations-to-threshold $\tau$ (the first iteration whose per-iteration mean reward reaches $0.9$) and the paired-seed speedup $\tau_{\mathrm{traj}}/\tau_{\mathrm{turn}}$, aggregated as above. We fit a power law to the speedup to measure its scaling; \textbf{$\alpha = -1/2$ is not expected here}, as these are different quantities from the SNR. AUC results, sweep grids, and seed counts are in Appendices~\ref{app:supplementary-auc} and~\ref{app:grids}.

\section{Results}
\label{sec:results}

\subsection{Experiment 1: gradient SNR scaling}
\label{sec:exp-init-snr}

The initialization SNR ratio is reported in Figure~\ref{fig:snr-vs-rho}. Fitting $R_{\mathrm{SNR}}(\rho)=C\rho^\alpha$ gives $\hat{\alpha}=-0.5084$ with $95\%$ bootstrap confidence interval $[-0.5118,-0.5054]$ and log-space $R^2=0.9994$. The fitted exponent is within $0.0084$ of the predicted value $-1/2$; constraining $\alpha=-1/2$ gives a nearly identical fit, with $R^2=0.9991$ (bootstrap test details in Appendix~\ref{app:supplementary-auc}). Corresponding curves under varying critic noise are reported in Appendix~\ref{app:critic-frontier}.

\paragraph{Interpretation.} Compared to the trajectory-level estimator, the perfect-critic turn-level estimator has $17.41$x higher SNR at $\rho=1/51$, and only $2.30$x\footnote{The ratio is not equal to $1$ at $\rho=1$ because the perfect-critic turn-level estimator has lower variance than the trajectory-level estimator even when no routine states are inserted.} SNR at $\rho=1$. These results support a causal effect of decision density on the gradient SNR ratio in this environment.

\vspace{-8pt}

\begin{figure}[h]
\centering
\safeincludegraphics[width=\linewidth]{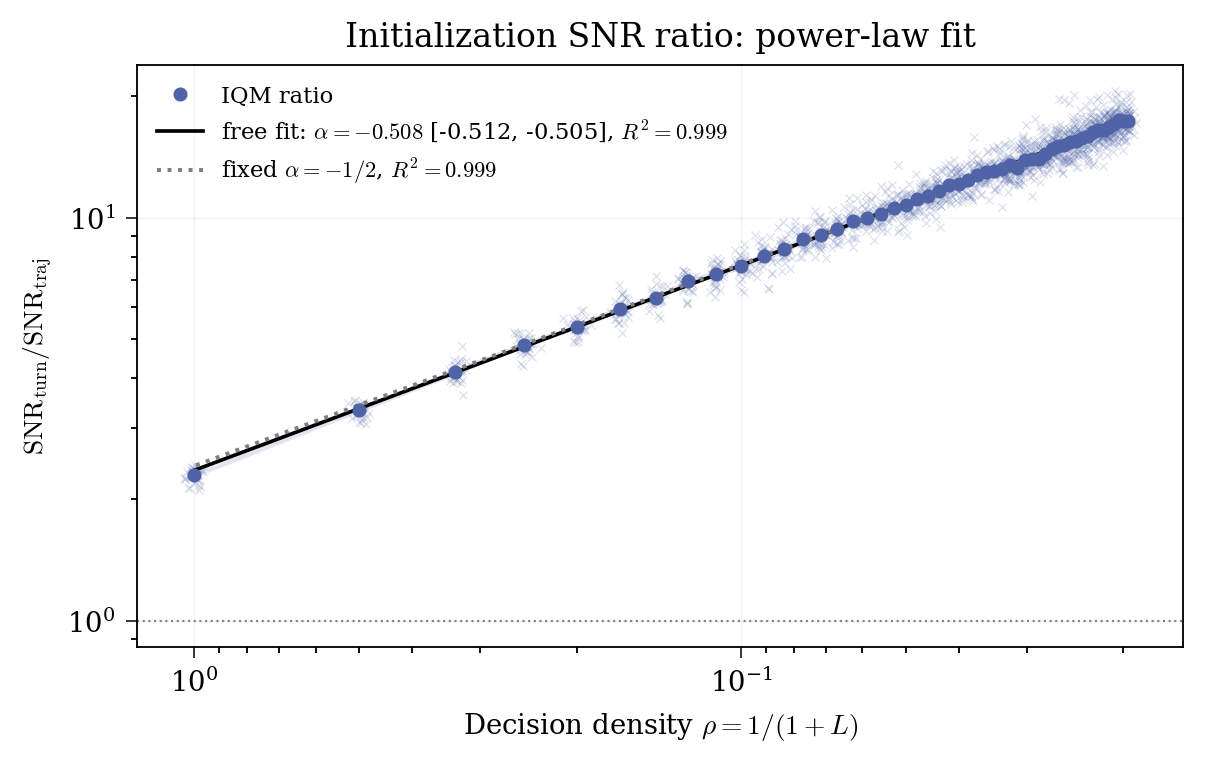}
\caption{Paired initialization SNR ratio $\mathrm{SNR}_{\mathrm{turn}}/\mathrm{SNR}_{\mathrm{traj}}$ vs.\ decision density, log scale. Points are per-$L$ IQMs.}
\label{fig:snr-vs-rho}
\end{figure}

\subsection{Experiment 2: training-step efficiency}
\label{sec:exp-training}

The threshold-speedup ratio as a function of $\rho$ is reported in Figure~\ref{fig:training-ratio-power}. At $\rho=1$, the two estimators reach the performance threshold in approximately the same number of training steps (paired-seed IQM $1.016$, interval $[0.961, 1.081]$). At $\rho=1/51$, the trajectory-level estimator takes roughly twice as many training steps as the turn-level estimator (paired-seed IQM $2.139$, interval $[1.772, 2.525]$). The fit gives $\hat\alpha=-0.19$ ($R^2=0.948$), well separated from zero and confirming that training-step efficiency depends on $\rho$. AUC and individual iters curves are in App.~\ref{app:iters-curve},~\ref{app:supplementary-auc},~\ref{app:supplementary-association}.

\paragraph{Interpretation.} Reducing decision density while keeping the critical problem fixed widens the training-step efficiency gap between estimators, as predicted. These results support a causal effect of decision density on training-step efficiency in this environment.

\begin{figure}[h]
\centering
\safeincludegraphics[width=\linewidth]{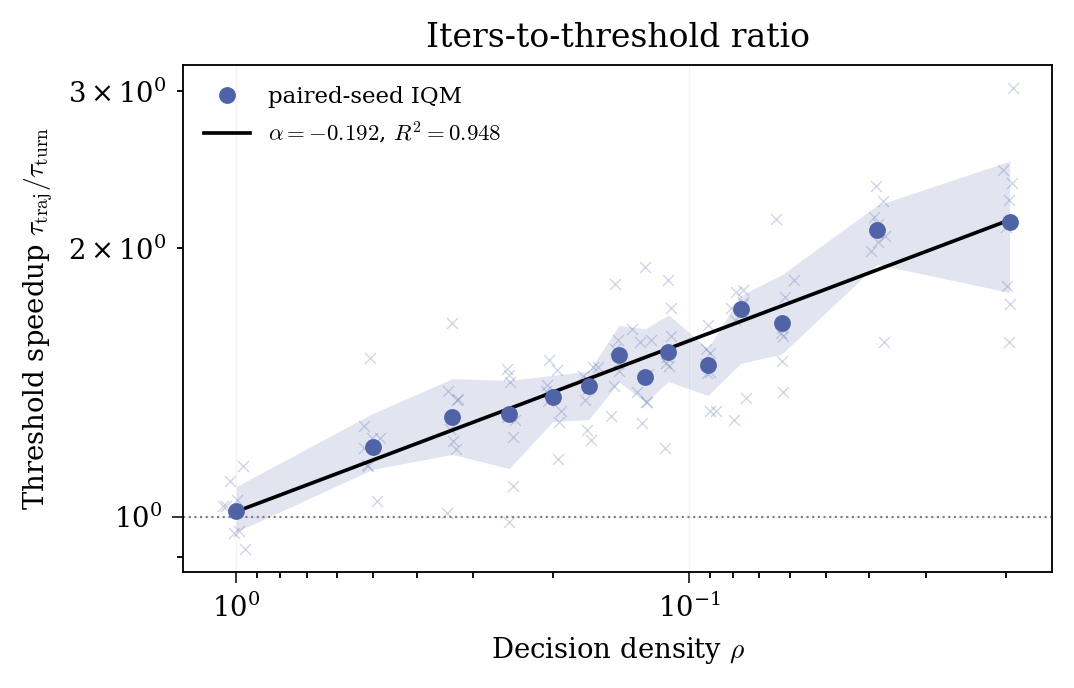}
\caption{Threshold speedup $\tau_{\mathrm{traj}}/\tau_{\mathrm{turn}}$ vs.\ decision density, log scale.}
\label{fig:training-ratio-power}
\end{figure}

\vspace{-8pt}

\section{Related Work}

\paragraph{Trajectory-level optimization for LLMs.}
Standard LLM post-training methods often assign supervision at the level of a complete response: RLHF for instruction following \citep{ouyang2022training}, DPO over preferred/dispreferred completions \citep{rafailov2023direct}, and GRPO with group-normalized trajectory rewards \citep{shao2024deepseekmath}. GRPO, introduced in DeepSeekMath, is especially relevant because it avoids training a value function by using group-level Monte Carlo baselines. Our work studies the statistical cost of trajectory-level credit assignment in multi-turn settings as a function of decision density.

\vspace{-4pt}

\paragraph{Turn-level credit assignment for LLM agents.}
Recent work shows that multi-turn agent RL benefits from credit assignment below the trajectory level. GiGPO \citep{feng2025gigpo} adds critic-free step-level relative advantages via anchor-state grouping, while turn-level reward design \citep{wei2025reinforcing} introduces dense turn-level rewards for multi-turn search agents. Turn-PPO \citep{li2026turnppo} further diagnoses direct multi-turn GRPO as unstable, attributing failures to uniform advantages across heterogeneous turns and high-variance sample-based estimates. These studies provide empirical evidence and algorithmic remedies for the limitations of trajectory-level credit assignment. Our contribution is complementary: we isolate \emph{signal dilution} as a statistical mechanism behind this gap and show how its SNR cost scales with decision density $\rho$.

\vspace{-4pt}

\paragraph{Gradient signal-to-noise ratio.}
Roberts and Tedrake \citep{roberts2008snr} argue that policy-gradient variance should be measured relative to the useful gradient signal, not in isolation. They show empirically that policy-gradient SNR predicts learning performance under fixed evaluation budgets. We use the same lens to study a different source of noise: routine turns that add trajectory-level gradient variance without adding expected signal. Our parameter-wise SNR is the square root of the parameter-wise GSNR used in supervised learning \citep{liu2020gsnr}.

\section{Discussion}
\label{sec:discussion}

We trace the turn-level versus trajectory-level gap to a single mechanism. Routine turns (Definition~\ref{def:routine}) contribute no expected signal: the perfect-critic turn-level estimator removes their contribution pointwise, while the trajectory-level estimator cancels it only in expectation (Proposition~\ref{prop:signal-equiv}). Under Assumptions~\ref{as:a1}--\ref{as:a2}, the residual variance accumulates linearly in the routine-turn count, yielding the $\Theta(\rho^{-1/2})$ SNR advantage of Proposition~\ref{prop:structural-scaling}. Proposition~\ref{prop:imperfect-critic-scaling} shows that gradient-unbiased critic error reintroduces routine-turn variance into the turn-level estimator, so the $\rho$-scaling depends on the relative magnitudes of decision density and critic-error variance. Diluted Doors recovers the predicted scaling under an intervention on $\rho$, and training efficiency gaps widen as $\rho$ decreases.

The $\rho^{-1/2}$ law should be read as an idealized upper envelope: real environments weaken the turn-level advantage through approximately routine states, near-deterministic policies on routine actions, and learned-critic error. The analysis nuances prior findings on turn-level credit assignment~\citep{feng2025gigpo, li2026turnppo, wei2025reinforcing}: the gap should track decision density rather than horizon alone, so at high $\rho$, even on long horizons, the cost of a critic is unlikely to be justified. Decision density thus serves both as a diagnostic for when a critic is worth its cost and as a target for algorithm design that addresses signal dilution without paying the full critic price.

Three limitations bound these claims. First, no experiments are run on real LLM agents; the upper-envelope reading should carry over since the mechanism is environment-independent, but the realized gap and the magnitude of $\rho$ in such systems have not been measured. Second, routine states are policy-dependent and not directly observable. Decision density can still be used conceptually, and a quantitative route is to probe a base agent through rollouts and estimate which action alternatives shift the downstream return distribution. Third, extending the definition to approximately routine states, e.g.\ by thresholding a distance between conditional return distributions, is the most important theoretical gap left open.

Several directions follow. The most direct is to estimate decision density, SNR ratios, and training gaps on agentic benchmarks with learned critics. A second is a soft-critical-state theory grounded in approximate distributional invariance, connected to the imperfect-critic regime. The mechanism also suggests critic-free variants that suppress routine-turn noise once routine states can be detected: lowering sampling temperature at these turns, or masking their score contributions.

\section{Conclusion}

Our findings support that multi-turn credit assignment is governed not by horizon alone, but by \emph{decision density}: the fraction of turns whose actions affect the return distribution. When $\rho$ is low, trajectory-level estimators attach a scalar advantage to many routine turns, adding variance without signal; turn-level estimators avoid this by cancelling routine contributions pointwise rather than only in expectation. This signal dilution mechanism yields an upper-envelope $\Theta(\rho^{-1/2})$ SNR advantage for turn-level credit assignment. The practical message is diagnostic: at high $\rho$, GRPO-like methods can avoid critic cost while remaining competitive; turn-level credit assignment becomes most useful when consequential decisions are sparse and critic error stays below the decision-density scale. A natural next step is to measure $\rho$ in real agentic tasks and design cheaper ways to suppress routine-turn noise.

\section*{Impact Statement}
This paper presents work whose goal is to advance the field of Machine Learning. There are many potential societal consequences of our work, none which we feel must be specifically highlighted here.

\bibliographystyle{icml2026}

\appendix

\section{Structural existence of a gap}
\label{app:structural-proof}
\begin{proof}[Proof of Lemma~\ref{lem:critical-routine-bounds}]
Fix the coordinate $k$ and suppress it. Decompose the trajectory-level estimator into critical and routine contributions:
\[
  \hat g_{\mathrm{traj}}=\hat A\,\Phi_C + \hat A\,\Phi_R .
\]

By Proposition~\ref{prop:signal-equiv}, the routine part of the trajectory-level estimator has zero mean:
\[
  \E[\hat A\,\Phi_R]
  =
  0.
\]
By A2, $\hat A$ is $\mathcal H_C$-measurable, so $\hat A^2$ is $\mathcal H_C$-measurable. Using the tower property and the two-sided A2(c),
\[
  \begin{aligned}
    \mathrm{Var}[\hat A\,\Phi_R]
    &= \E[\hat A^2\Phi_R^2]\\
    &= \E\!\left[\hat A^2\,\E[\Phi_R^2\mid\mathcal H_C]\right],
  \end{aligned}
\]
\[
  \begin{aligned}
    c_R\,\E[\hat A^2 N_R]
    &\le \mathrm{Var}[\hat A\,\Phi_R]\\
    &\le C_R\,\E[\hat A^2 N_R].
  \end{aligned}
\]
For the lower bound, A2(b) and A2(a) give $\E[\hat A^2 N_R]\ge c_A\,\E[N_R]\ge c_A c_N\, T$. For the upper bound, $|\hat A|\le M_A$ by the bounded-advantage assumption and $N_R\le T$ give $\E[\hat A^2 N_R]\le M_A^2\, T$. Hence
\[
  \begin{aligned}
    c_R c_A c_N\, T
    &\le \mathrm{Var}[\hat A\,\Phi_R]
    \le C_R M_A^2\, T,\\
    \text{i.e.,}\qquad
    \mathrm{Var}[\hat A\,\Phi_R]
    &= \Theta(T).
  \end{aligned}
\]

By Prop.~\ref{prop:signal-equiv}, the expected turn-level and trajectory-level gradients reduce to their critical-state contributions, which A1 bounds away from $0$ and $\infty$:
\[
  m_\mu \le |\E[\hat g_{\mathrm{turn}}]| \le M_\mu,
  \qquad
  m_\mu \le |\E[\hat g_{\mathrm{traj}}]| \le M_\mu .
\]

It remains to compare variances. By A1,
\[
  m_C\;\le\;\mathrm{Var}[\hat g_{\mathrm{turn}}]\le M_C,
  \qquad
  \mathrm{Var}[\hat A\,\Phi_C]\le M_C,
\]
while $\mathrm{Var}[\hat A\,\Phi_R]=\Theta(T)$ from the previous display. Writing the standard deviation of the sum as an $L^2$ norm of the centered variables,
\[
  \begin{aligned}
    &\sqrt{\mathrm{Var}[\hat A\,\Phi_C+\hat A\,\Phi_R]}\\
    &\qquad =
    \left\lVert
      (\hat A\,\Phi_C - \E[\hat A\,\Phi_C])
      + (\hat A\,\Phi_R - \E[\hat A\,\Phi_R])
    \right\rVert_2,
  \end{aligned}
\]
For sufficiently large $T$, the reverse and forward triangle inequalities, together with $c_R c_A c_N\, T\le \mathrm{Var}[\hat A\,\Phi_R]\le C_R M_A^2\, T$ and $\mathrm{Var}[\hat A\,\Phi_C]\le M_C$, bound it on both sides:
\[
  \begin{aligned}
    \sqrt{c_R c_A c_N\, T} - \sqrt{M_C}
    &\le
    \sqrt{\mathrm{Var}[\hat A\,\Phi_C+\hat A\,\Phi_R]}\\
    &\le
    \sqrt{C_R M_A^2\, T} + \sqrt{M_C}\\
    &= \Theta(\sqrt T).
  \end{aligned}
\]
Hence $\mathrm{Var}[\hat g_{\mathrm{traj}}]=\Theta(T)$. The one-coordinate bounds are therefore
\[
  \begin{aligned}
    |\E[\hat g_{\mathrm{turn}}]|
    &= \Theta(1),\\
    |\E[\hat g_{\mathrm{traj}}]|
    &= \Theta(1),\\
    \mathrm{Var}[\hat g_{\mathrm{turn}}]
    &= \Theta(1),\\
    \mathrm{Var}[\hat g_{\mathrm{traj}}]
    &= \Theta(T).
  \end{aligned}
\]
Inspecting the derivations above, the upper bounds use only the upper-bound parts of A1 and the upper half of A2(c), while the lower bounds additionally invoke A1's lower-bound parts together with A2(a), A2(b), and the lower half of A2(c).
\end{proof}

\begin{proof}[Proof of Proposition~\ref{prop:structural-scaling}]
Combining the bounds from Lemma~\ref{lem:critical-routine-bounds},
\[
  \begin{aligned}
    \frac{\mathrm{SNR}_k(\hat g_{\mathrm{turn}})}
         {\mathrm{SNR}_k(\hat g_{\mathrm{traj}})}
    &=
    \frac{|\E[\hat g_{\mathrm{turn}}]|}{|\E[\hat g_{\mathrm{traj}}]|}\\
    &\quad{}\times
    \sqrt{\frac{\mathrm{Var}[\hat g_{\mathrm{traj}}]}
               {\mathrm{Var}[\hat g_{\mathrm{turn}}]}}\\
    &=
    \Theta(\sqrt T) = \Theta(\rho^{-1/2}).
  \end{aligned}
\]
\end{proof}

\subsection{Pooled SNR scaling}
\label{app:pooled-proof}

\begin{proof}[Proof of Corollary~\ref{cor:pooled-snr}]
Let $\mathcal K=\{1,\ldots,d\}$ and write
\[
  \begin{aligned}
    \mu_{\bullet,k}
    &= \E[\hat g_{\bullet,k}],\\
    \sigma_{\bullet,k}^2
    &= \mathrm{Var}[\hat g_{\bullet,k}],\\
    \bullet
    &\in\{\mathrm{turn},\mathrm{traj}\}.
  \end{aligned}
\]
By Lemma~\ref{lem:critical-routine-bounds}, the uniform upper conditions give, for every $k\in\mathcal K$,
\[
  \begin{aligned}
    |\mu_{\bullet,k}|
    &= O(1),\\
    \sigma_{\mathrm{turn},k}^2
    &= O(1),\\
    \sigma_{\mathrm{traj},k}^2
    &= O(T),
  \end{aligned}
\]
while the lower conditions at $k^\star$ give
\[
  \begin{aligned}
    |\mu_{\bullet,k^\star}|
    &= \Omega(1),\\
    \sigma_{\mathrm{turn},k^\star}^2
    &= \Omega(1),\\
    \sigma_{\mathrm{traj},k^\star}^2
    &= \Omega(T).
  \end{aligned}
\]
Since $d$ is fixed,
\[
  \begin{aligned}
    \|\E\hat g_\bullet\|^2
    &=
    \sum_{k\in\mathcal K}\mu_{\bullet,k}^2
    =
    \Theta(1),\\
    \sum_{k\in\mathcal K}\sigma_{\mathrm{turn},k}^2
    &=
    \Theta(1),\\
    \sum_{k\in\mathcal K}\sigma_{\mathrm{traj},k}^2
    &=
    \Theta(T).
  \end{aligned}
\]
Substituting into Definition~\ref{def:snr-pooled},
\[
  \begin{aligned}
    \frac{\mathrm{SNR}(\hat g_{\mathrm{turn}})}
         {\mathrm{SNR}(\hat g_{\mathrm{traj}})}
    &=
    \frac{\|\E\hat g_{\mathrm{turn}}\|}
         {\|\E\hat g_{\mathrm{traj}}\|}\\
    &\quad{}\times
    \sqrt{\frac{\sum_k\sigma_{\mathrm{traj},k}^2}
               {\sum_k\sigma_{\mathrm{turn},k}^2}}\\
    &=
    \Theta(1)\cdot \sqrt{\frac{\Theta(T)}{\Theta(1)}}\\
    &=
    \Theta(\sqrt T)
    =
    \Theta(\rho^{-1/2}).
  \end{aligned}
\]
\end{proof}

\subsection{Imperfect critic proof}
\label{app:imperfect-critic-proof}

\begin{proof}[Proof of Proposition~\ref{prop:imperfect-critic-scaling}]
Write
\[
  \hat g_{\mathrm{turn},\epsilon,k}
  =
  \hat g_{\mathrm{turn},k}^\star
  +
  Z_{\epsilon,k}.
\]
By Assumption~\ref{as:a3},
\[
  \E[Z_{\epsilon,k}\mid\mathcal F_\tau]=0.
\]
Since $\hat g_{\mathrm{turn},k}^\star$ is determined by the sampled trajectory,
\[
  \E[
    \hat g_{\mathrm{turn},\epsilon,k}
    \mid
    \mathcal F_\tau
  ]
  =
  \hat g_{\mathrm{turn},k}^\star.
\]
Taking expectation gives
\[
  \E[
    \hat g_{\mathrm{turn},\epsilon,k}
  ]
  =
  \E[
    \hat g_{\mathrm{turn},k}^\star
  ].
\]
By Lemma~\ref{lem:critical-routine-bounds}, under Assumptions~\ref{as:a1}--\ref{as:a2},
\[
  |\E[\hat g_{\mathrm{turn},k}^\star]|
  =
  \Theta(1),
\]
so
\[
  |\E[\hat g_{\mathrm{turn},\epsilon,k}]|
  =
  \Theta(1).
\]

For the variance, apply the law of total variance:
\[
  \mathrm{Var}[
    \hat g_{\mathrm{turn},\epsilon,k}
  ]
  =
  \mathrm{Var}
  \left(
    \E[
      \hat g_{\mathrm{turn},\epsilon,k}
      \mid
      \mathcal F_\tau
    ]
  \right)
\]
\[
  +\E
  \left[
    \mathrm{Var}
    \left(
      \hat g_{\mathrm{turn},\epsilon,k}
      \mid
      \mathcal F_\tau
    \right)
  \right].
\]
The first term is
\[
  \mathrm{Var}
  \left(
    \E[
      \hat g_{\mathrm{turn},\epsilon,k}
      \mid
      \mathcal F_\tau
    ]
  \right)
  =
  \mathrm{Var}[
    \hat g_{\mathrm{turn},k}^\star
  ].
\]
Again by Lemma~\ref{lem:critical-routine-bounds},
\[
  \mathrm{Var}[
    \hat g_{\mathrm{turn},k}^\star
  ]
  =
  \Theta(1).
\]
For the second term, $\hat g_{\mathrm{turn},k}^\star$ is fixed conditional on $\mathcal F_\tau$, so
\[
  \mathrm{Var}
  \left(
    \hat g_{\mathrm{turn},\epsilon,k}
    \mid
    \mathcal F_\tau
  \right)
  =
  \mathrm{Var}
  \left(
    Z_{\epsilon,k}
    \mid
    \mathcal F_\tau
  \right).
\]
By the definition of $\sigma_{\epsilon,k}^2(T)$ in Assumption~\ref{as:a3},
\[
  \E
  \left[
    \mathrm{Var}
    \left(
      Z_{\epsilon,k}
      \mid
      \mathcal F_\tau
    \right)
  \right]
  =
  T\sigma_{\epsilon,k}^2(T).
\]
Therefore,
\[
  \mathrm{Var}[
    \hat g_{\mathrm{turn},\epsilon,k}
  ]
  =
  \Theta(1)
  +
  T\sigma_{\epsilon,k}^2(T)
  =
  \Theta\!\left(
    1+T\sigma_{\epsilon,k}^2(T)
  \right).
\]

The existing structural result gives
\[
  |\E[\hat g_{\mathrm{traj},k}]|
  =
  \Theta(1),
  \qquad
  \mathrm{Var}[
    \hat g_{\mathrm{traj},k}
  ]
  =
  \Theta(T).
\]
Thus
\[
  \begin{aligned}
  \frac{
    \mathrm{SNR}_k(\hat g_{\mathrm{turn},\epsilon})
  }{
    \mathrm{SNR}_k(\hat g_{\mathrm{traj}})
  }
  &=
  \frac{
    |\E[\hat g_{\mathrm{turn},\epsilon,k}]|
  }{
    |\E[\hat g_{\mathrm{traj},k}]|
  }
  \sqrt{
    \frac{
      \mathrm{Var}[
        \hat g_{\mathrm{traj},k}
      ]
    }{
      \mathrm{Var}[
        \hat g_{\mathrm{turn},\epsilon,k}
      ]
    }
  } \\
  &=
  \Theta(1)
  \sqrt{
    \frac{
      \Theta(T)
    }{
      \Theta(1+T\sigma_{\epsilon,k}^2(T))
    }
  } \\
  &=
  \Theta\!\left(
    \sqrt{
      \frac{T}
      {1+T\sigma_{\epsilon,k}^2(T)}
    }
  \right).
  \end{aligned}
\]
Finally, using $T=\bar C/\rho$,
\[
  \sqrt{
    \frac{T}
    {1+T\sigma_{\epsilon,k}^2(T)}
  }
  =
  \sqrt{
    \frac{\bar C}
    {\rho+\bar C\sigma_{\epsilon,k}^2(T)}
  }
  =
  \Theta\!\left(
    \frac{1}
    {\sqrt{\rho+\sigma_{\epsilon,k}^2(T)}}
  \right),
\]
since $\bar C$ is fixed.
\end{proof}

\subsection{Asymptotic form}
\label{app:scaling-asymptotic}

Under the hypotheses of Proposition~\ref{prop:structural-scaling},
\begin{equation}
  \begin{aligned}
    \frac{\mathrm{SNR}_{\mathrm{turn}}}
         {\mathrm{SNR}_{\mathrm{traj}}}
    &=
    \Theta(\sqrt T)
    =
    \Theta(\rho^{-1/2}),
  \end{aligned}
  \label{eq:scaling-master}
\end{equation}
in the limit $T \to \infty$ at fixed $\bar C$.

\section{Diluted Doors environment and training details}
\label{app:experiments}

This appendix expands the experimental setup of Section~\ref{sec:setup}.

\subsection{Environment design}

Diluted Doors is a deterministic tree MDP. The agent observes its current state (position in the tree and available actions) and chooses one action per turn, selecting a door at choice states or a path at routine states. The tree structure and reward function are fixed across episodes; the agent must discover the reward structure through repeated interaction.

\paragraph{Reward.} The reward is sparse and additive, revealed only at the end of the episode and equal to the fraction of correct choices at critical states:
\[
  R = \frac{1}{C} \sum_{c=1}^{C} \mathbf{1}[\text{agent chose the correct door at depth } c].
\]
The additive form ensures that a wrong choice at depth $c$ does not eliminate subsequent critical states: every trajectory visits exactly $C$ critical states regardless of actions taken.

\paragraph{Tree structure.} Every root-to-leaf path has length $T = C(1+L)$ with $C$ choice states and $C\cdot L$ routine states. Because the tree is balanced and routine subtrees are isomorphic, every trajectory visits exactly $C$ critical states, so $\bar C = C$ and $\rho = 1/(1+L)$ exactly.

\subsection{Concrete episode}

For $C=2$, $K_C=2$, $K_R=2$, $L=1$ ($T=4$, $\rho=1/2$), with correct doors $\{1, 2\}$ at depths $\{1, 2\}$:

\begin{center}
\small
\resizebox{\columnwidth}{!}{%
\begin{tabular}{clll}
\toprule
\textbf{Turn} & \textbf{Type} & \textbf{State shows} & \textbf{Agent picks} \\
\midrule
1 & Critical & \texttt{door\_1, door\_2} & \texttt{door\_1} (correct) \\
2 & Routine & \texttt{path\_1, path\_2} & \texttt{path\_2} (irrelevant) \\
3 & Critical & \texttt{door\_1, door\_2} & \texttt{door\_1} (wrong; correct is \texttt{door\_2}) \\
4 & Routine & \texttt{path\_1, path\_2} & \texttt{path\_1} (irrelevant) \\
\midrule
\multicolumn{4}{l}{\textbf{Terminal reward:} $R = (1 + 0)/2 = 1/2$.} \\
\bottomrule
\end{tabular}
}
\end{center}

\subsection{Tunable parameters}

\begin{center}
\resizebox{\columnwidth}{!}{%
\begin{tabular}{ll}
\toprule
\textbf{Parameter} & \textbf{} \\
\midrule
$C$ & Number of critical states \\
$L$ & Routine states per critical state \\
$K_C$ & Doors per critical state \\
$K_R$ & Paths per routine state \\
$T = C(1+L)$ & Total turns \\
$\rho = 1/(1+L)$ & Decision density \\
\bottomrule
\end{tabular}
}
\end{center}

\subsection{Policy class and initialization}
\label{app:policy-class}

The policy $\pi_\theta$ is a small causal Transformer \citep{vaswani2017attention} ($30$k
parameters). Actions are represented as tokens: the vocabulary has
$K{+}1$ entries, one per action plus a \texttt{START} symbol. At
turn $t$, the input sequence is
$[\texttt{START}, a_0, \ldots, a_{t-1}]$, and a linear head on the
final position produces a distribution over the $K$ next actions.

The network uses two pre-norm Transformer blocks with model
dimension $32$, two-head self-attention, and a feed-forward
sub-block of hidden width $128$. Token and positional embeddings
are learned.

Weight initialization is Xavier-normal with gain $0.1$ on linear layers and $\mathcal{N}(0, 0.02^2)$ on embeddings. At initialization, the policy is close to uniform across all states: the normalized action entropy
$H(\pi(\cdot{\mid}s))/\log K$ is close to $1$ and the mean
per-step KL divergence to the uniform distribution is close to $0$.

\subsection{On-policy training loop}

One rollout group of $G$ trajectories drives one gradient step. We use the simplified gradients of Section~\ref{sec:estimators} directly: no importance ratio, no clipping, and no KL penalty. The two estimators differ only in the advantage plugged into Eq.~\eqref{eq:ppo-grad} or Eq.~\eqref{eq:grpo-grad}.

\paragraph{Trajectory-level estimator.} We use the group-normalized advantage of Eq.~\eqref{eq:grpo-advantage}, broadcast to every turn of trajectory $i$. A single gradient step is taken per rollout group; multiple epochs would require an importance-ratio correction whose interaction with the clip is a separate question from the gradient-form one we isolate.

\paragraph{Turn-level estimator with Monte~Carlo $V^\pi$ critic.} For each visited state $s_t^{(i)}$, we draw $k$ fresh rollouts under the current policy and set $\hat V(s_t^{(i)})$ to the mean of their returns. Under deterministic transitions and terminal reward, the TD(0) advantage
\begin{equation}
  \begin{aligned}
    \hat A_t
    &= \hat V(s_{t+1}) - \hat V(s_t)
    \quad (t < T-1),\\
    \hat A_{T-1}
    &= R - \hat V(s_{T-1})
  \end{aligned}
  \label{eq:td0-mc}
\end{equation}
is consistent for $A^{\pi_\theta}$, so at routine states $\hat A_t \to 0$ as $k \to \infty$, matching Proposition~\ref{prop:signal-equiv}(ii). Since $\hat V(s_t)$ and $\hat V(s_{t+1})$ are computed from disjoint rollout batches, they are independent and the advantage variance scales as
\[
  \mathrm{Var}[\hat A_t]
  \;=\;
  \frac{\mathrm{Var}_\tau[R(\tau) \mid s_{t+1}, \pi_\theta] + \mathrm{Var}_\tau[R(\tau) \mid s_t, \pi_\theta]}{k}.
\] We use Monte Carlo rather than a learned critic so that critic quality is a known, controllable quantity ($k$); a learned critic would conflate the structural gap with critic-training dynamics.

\paragraph{Optimizer.} AdamW \citep{loshchilov2019decoupled} with learning rate $10^{-3}$, linear warmup over the first $10\%$ of iterations, gradient clipping at $1.0$ \citep{pascanu2013difficulty}, no entropy regularization. Each run comprises $500$ iterations.

\subsection{Evaluation statistics}
\label{app:stats}

For each $L$, scalar metrics are aggregated across seeds using the interquartile mean (IQM) \citep{agarwal2021deep}. Confidence intervals are percentile bootstrap intervals with $5000$ resamples \citep{efron1993bootstrap}. For paired ratios, the ratio is computed seed-by-seed before aggregating across seeds.

Given per-$L$ aggregate ratios $R(\rho)$, we fit
\[
  R(\rho) = C \rho^\alpha
\]
by ordinary least squares. Reported $R^2$ values are computed in log-space. Confidence intervals for $\alpha$ are obtained by bootstrapping seeds within each $L$, recomputing the per-$L$ aggregate ratio, and refitting the power law.

For the association analysis, the common $L$ grid between the initialization and training sweeps is used. We compute Pearson correlation between log initialization SNR ratios and log training ratios. Confidence intervals are bootstrap intervals over the common $L$ values; permutation $p$-values are two-sided.

\subsection{Supplementary association figures}
\label{app:supplementary-association}

Figure~\ref{fig:snr-training-correlation} reports the association between the initialization SNR ratio and the threshold-speedup ratio.
Figure~\ref{fig:snr-auc-correlation} reports the association between the initialization SNR ratio and the AUC ratio.

\begin{figure}[h]
\centering
\safeincludegraphics[width=\linewidth]{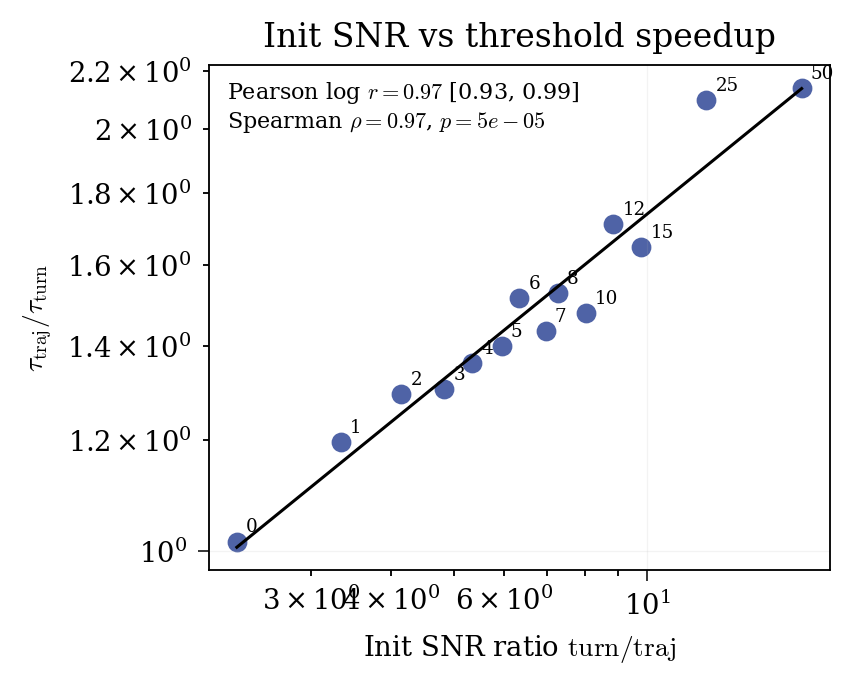}
\caption{Association between the initialization SNR ratio and threshold-speedup ratio on the shared $L$ values. Each point is one value of $L$.}
\label{fig:snr-training-correlation}
\end{figure}

\begin{figure}[h]
\centering
\safeincludegraphics[width=\linewidth]{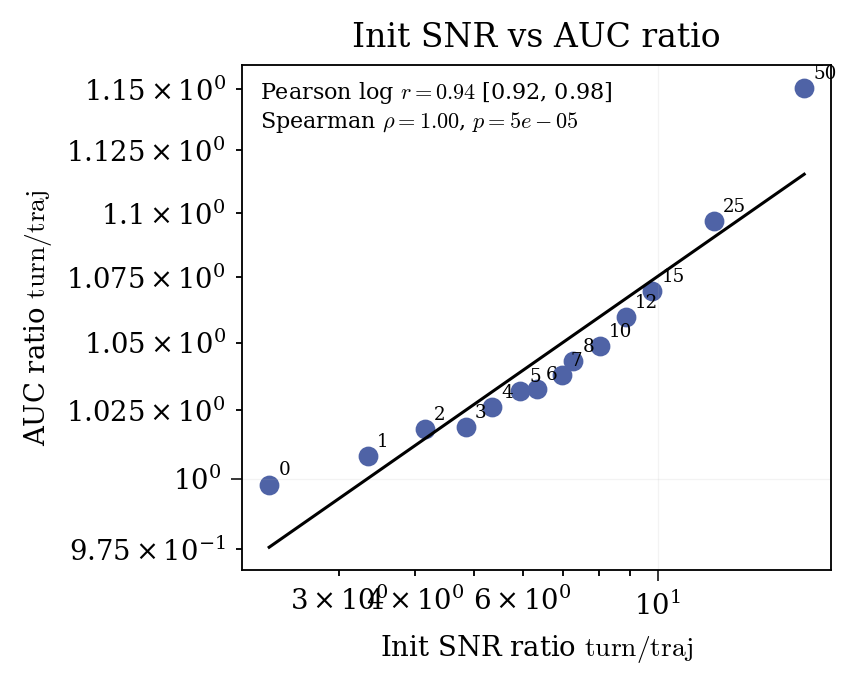}
\caption{Association between the initialization SNR ratio and AUC ratio on the shared $L$ values. Each point is one value of $L$.}
\label{fig:snr-auc-correlation}
\end{figure}

\subsection{Per-estimator iterations-to-threshold}
\label{app:iters-curve}

Figure~\ref{fig:iters-to-target} reports iterations-to-threshold for each estimator separately, complementing the speedup ratio in the main text. At $\rho=1$, both estimators reach the threshold in roughly $76$--$78$ iterations; at $\rho=1/51$, the trajectory-level estimator requires $225.25$ iterations ($[186.0, 254.0]$) while the turn-level estimator requires $102.0$ ($[94.75, 118.0]$).

\begin{figure}[h]
\centering
\safeincludegraphics[width=\linewidth]{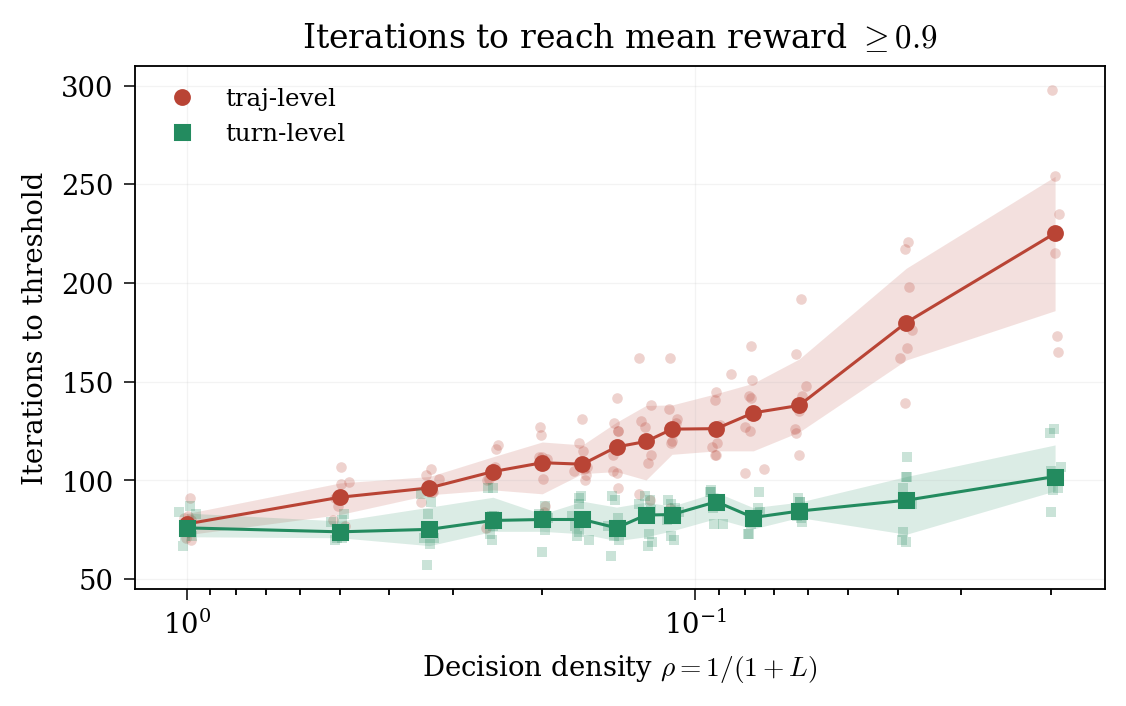}
\caption{Iterations to reach mean reward at least $0.9$ vs.\ decision density on a logarithmic $\rho$ axis.}
\label{fig:iters-to-target}
\end{figure}

\subsection{Critic-noise frontier}
\label{app:critic-frontier}

The initialization SNR ratio under varying levels of critic noise empirically traces the imperfect-critic frontier of Proposition~\ref{prop:imperfect-critic-scaling}. Figure~\ref{fig:critic-frontier-L} shows the SNR ratio as a function of $L$ for several critic-noise levels; Figure~\ref{fig:critic-frontier-noise} shows the same data as a function of critic noise for several values of $L$.

\begin{figure}[h]
\centering
\safeincludegraphics[width=\linewidth]{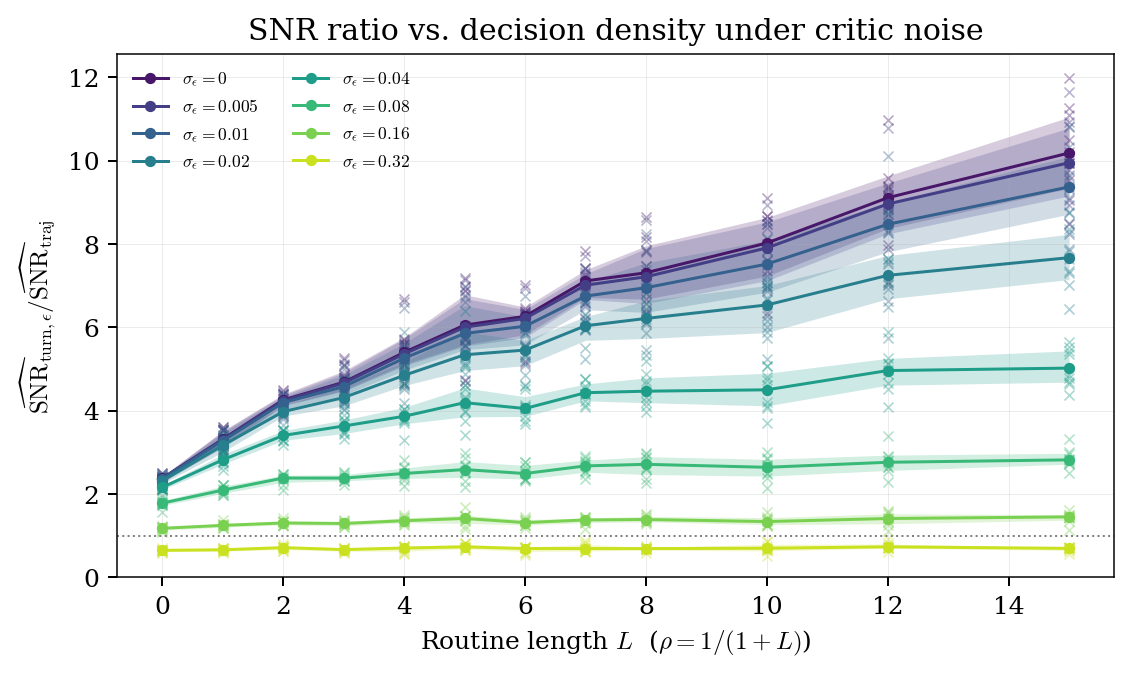}
\caption{Initialization SNR ratio $\mathrm{SNR}_{\mathrm{turn},\epsilon}/\mathrm{SNR}_{\mathrm{traj}}$ vs.\ $L$ at several critic-noise levels. Horizontal asymptotes form as we raise the critic error scale.}
\label{fig:critic-frontier-L}
\end{figure}

\begin{figure}[h]
\centering
\safeincludegraphics[width=\linewidth]{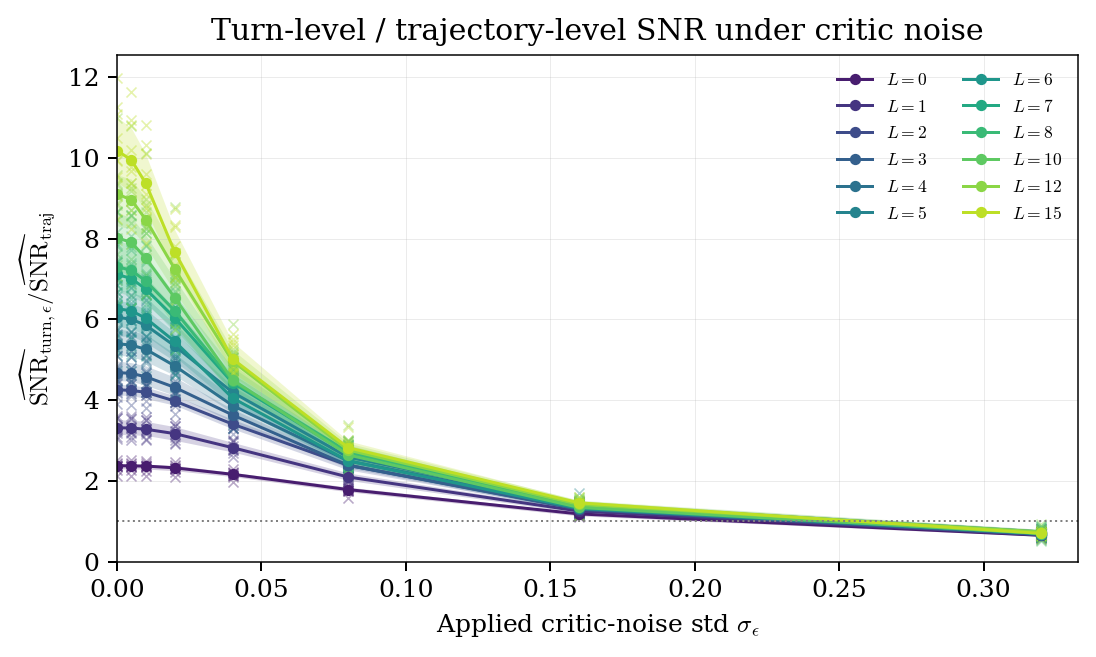}
\caption{Initialization SNR ratio $\mathrm{SNR}_{\mathrm{turn},\epsilon}/\mathrm{SNR}_{\mathrm{traj}}$ vs.\ critic noise at several values of $L$.}
\label{fig:critic-frontier-noise}
\end{figure}

\subsection{Supplementary AUC figures}
\label{app:supplementary-auc}

AUC is retained as a whole-curve training summary, but the main analysis uses iterations-to-threshold. Figure~\ref{fig:auc-vs-rho} reports AUC by decision density, Figure~\ref{fig:auc-ratio-power} reports the paired AUC ratio, and Figure~\ref{fig:snr-auc-correlation} reports its association with the initialization SNR ratio.

\paragraph{Bootstrap test for $\alpha=-1/2$ (Experiment 1).} A bootstrap test of exact equality to $\alpha=-1/2$ rejects at $p=0.0004$; the estimated deviation is small ($|\hat{\alpha}+1/2|=0.0084$), and the full confidence interval remains inside the practical equivalence band $[-0.52,-0.48]$.

\begin{figure}[h]
\centering
\safeincludegraphics[width=\linewidth]{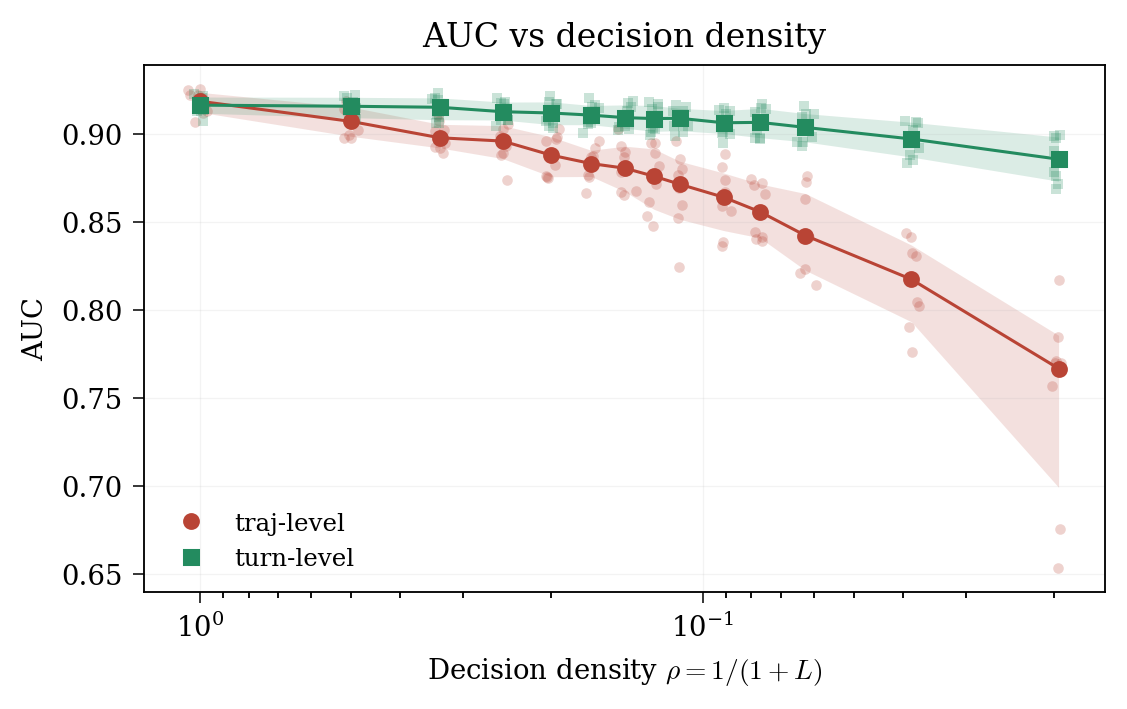}
\caption{AUC vs.\ decision density on a logarithmic $\rho$ axis.}
\label{fig:auc-vs-rho}
\end{figure}

\begin{figure}[h]
\centering
\safeincludegraphics[width=\linewidth]{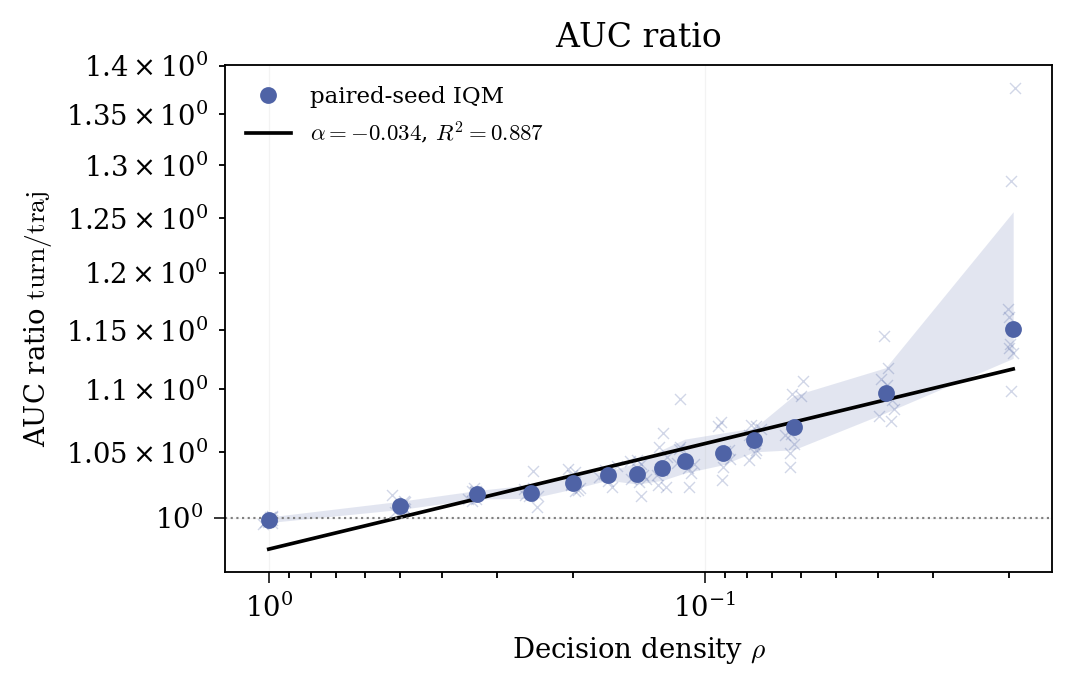}
\caption{AUC ratio $\mathrm{AUC}_{\mathrm{turn}}/\mathrm{AUC}_{\mathrm{traj}}$ vs.\ decision density.}
\label{fig:auc-ratio-power}
\end{figure}

\subsection{Sweep grids}
\label{app:grids}

All experiments share $C = K_C = K_R = 4$, so random-policy reward is $1/K_C = 0.25$ and the routine action space is non-trivial ($K_R > 1$); whether the policy actually places mass on more than one routine action, and hence whether routine score variance is nonzero, as required by Assumption~\ref{as:a2}(c), depends on $\pi_\theta$ and is not enforced by construction after initialization. Each experimental seed sets both the training run (weight initialization and on-policy rollout sampling) and the MDP instance (which of the $K_C$ doors is correct at each critical depth), so cross-seed confidence intervals reflect robustness to both simultaneously.

\begin{table}[h]
\centering
\small
\begin{tabular}{@{}l>{\raggedright\arraybackslash}p{0.62\columnwidth}@{}}
\toprule
\multicolumn{2}{@{}l}{\textbf{Training-performance sweep}}\\
\midrule
$L$ grid          & $\{0, 1, 2, 3, 4, 5, 6, 7, 8, 10, 12, 15, 25, 50\}$\\
Seeds / cell      & $8$\\
Hyperparameters   & Traj-level $G = 16$, Turn-level $k = 8$\\
\midrule
\multicolumn{2}{@{}l}{\textbf{SNR, init-probe}}\\
\midrule
$L$ grid          & $\{0, 1, \ldots, 50\}$\\
Seeds / cell      & $32$\\
Hyperparameters   & $G = 16$, $N = 1024$, analytical $V$\\
\bottomrule
\end{tabular}
\caption{Sweep grids. Decision density is $\rho = 1/(1+L)$, so the $L \leq 50$ range covers $\rho \in [\tfrac{1}{51}, 1]$.}
\label{tab:grids}
\end{table}

\end{document}